\documentclass[sigconf]{acmart}
\AtBeginDocument{%
  }

\usepackage{amsmath}
\usepackage{amssymb}
\usepackage{amsthm}
\usepackage{bm}
\usepackage{placeins}
\usepackage{array}
\usepackage{tabularx}
\usepackage{pifont}
\usepackage{colortbl}
\usepackage{booktabs} 
\usepackage{multirow}  
\usepackage{makecell} 
\DeclareMathOperator*{\argmax}{arg\,max}

\copyrightyear{2026}
\acmYear{2026}
\setcopyright{cc}
\setcctype{by}
\acmConference[MM '26]{Proceedings of the 34th ACM International Conference on Multimedia}{November 10--14, 2026}{Rio de Janeiro, Brazil}
\acmBooktitle{Proceedings of the 34th ACM International Conference on Multimedia (MM '26), November 10--14, 2026, Rio de Janeiro, Brazil}
\acmDOI{10.1145/3767308.3836429}
\acmISBN{979-8-4007-2213-4/2026/11}

\begin{document}

\title{Majorization-Guided Test-Time Adaptation for Vision-Language Models under Modality-Specific Shift}

\author{Lixian Chen}
\affiliation{%
  \institution{Guangdong University of Technology}
  \city{Guangzhou}
  \country{China}
}
\email{3123003175@mail2.gdut.edu.cn}

\author{Mingxuan Huang}
\affiliation{%
  \institution{Sun Yat-Sen University}
  \city{Guangzhou}
  \country{China}
}
\email{huangmx53@mail2.sysu.edu.cn}

\author{Yanhui Chen}
\affiliation{%
  \institution{Guangdong University of Technology}
  \city{Guangzhou}
  \country{China}
}
\email{chenyanhui91@mails.gdut.edu.cn}

\author{Junyi Lin}
\affiliation{%
  \institution{Guangdong University of Technology}
  \city{Guangzhou}
  \country{China}
}
\email{3123001378@mail2.gdut.edu.cn}

\author{Yang Shi}
\affiliation{%
  \institution{Guangdong University of Technology}
  \city{Guangzhou}
  \country{China}
}
\email{sudo.shiyang@gmail.com}

\renewcommand{\shortauthors}{Chen et al.}

\begin{abstract}
Vision--language models can face asymmetric visual and textual shifts
at deployment. These shifts expose a multimodal failure mode in which an
unreliable branch remains overconfident, dominates fusion, and causes
entropy-based test-time adaptation to sharpen an incorrect prediction.
We model this behavior as doubly stochastic posterior mixing and cast
adaptation as constrained de-mixing. Majorization-Guided Multimodal
Test-Time Adaptation (MG-MTTA) freezes
both encoders and updates only a lightweight fusion module. Running-anchor
consistency estimates relative branch drift, while cross-modal conflict
regulates modality dominance before entropy sharpening. The analysis gives
sufficient conditions for entropy reduction to preserve the clean decision
and an explicit threshold at which a biased modality reverses the fused
ranking. Across visual, textual, and joint shifts, MG-MTTA improves
ImageNet top-1 accuracy from 57.97\% to 66.51\% under textual shift and
from 21.68\% to 26.27\% under joint shift, while reducing
wrong-more-confident failures. The largest gains occur under textual and joint shifts, where the two
branches differ more in reliability. Project page:
\url{https://mg-mtta.github.io/}.
\end{abstract}



\maketitle

\section{Introduction}
\label{sec:introduction}

Vision--language models (VLMs) align visual and textual representations
at scale and support zero-shot transfer through natural-language prompts
~\cite{vlm_clip,vlm_siglip}. Flamingo extends this pretraining paradigm
to few-shot multimodal interaction~\cite{vlm_flamingo}. At deployment,
however, both branches face changing conditions: image corruption can
perturb visual evidence, while prompt wording, context, tokenization, or
metadata can alter textual evidence. Prompt learning improves task
adaptation~\cite{zhou2022learning}, but does not eliminate this test-time
variability.

\begin{figure*}[t]
\centering
\includegraphics[width=0.95\textwidth,trim={2.0cm 2.5cm 3.0cm 5.0cm},clip]{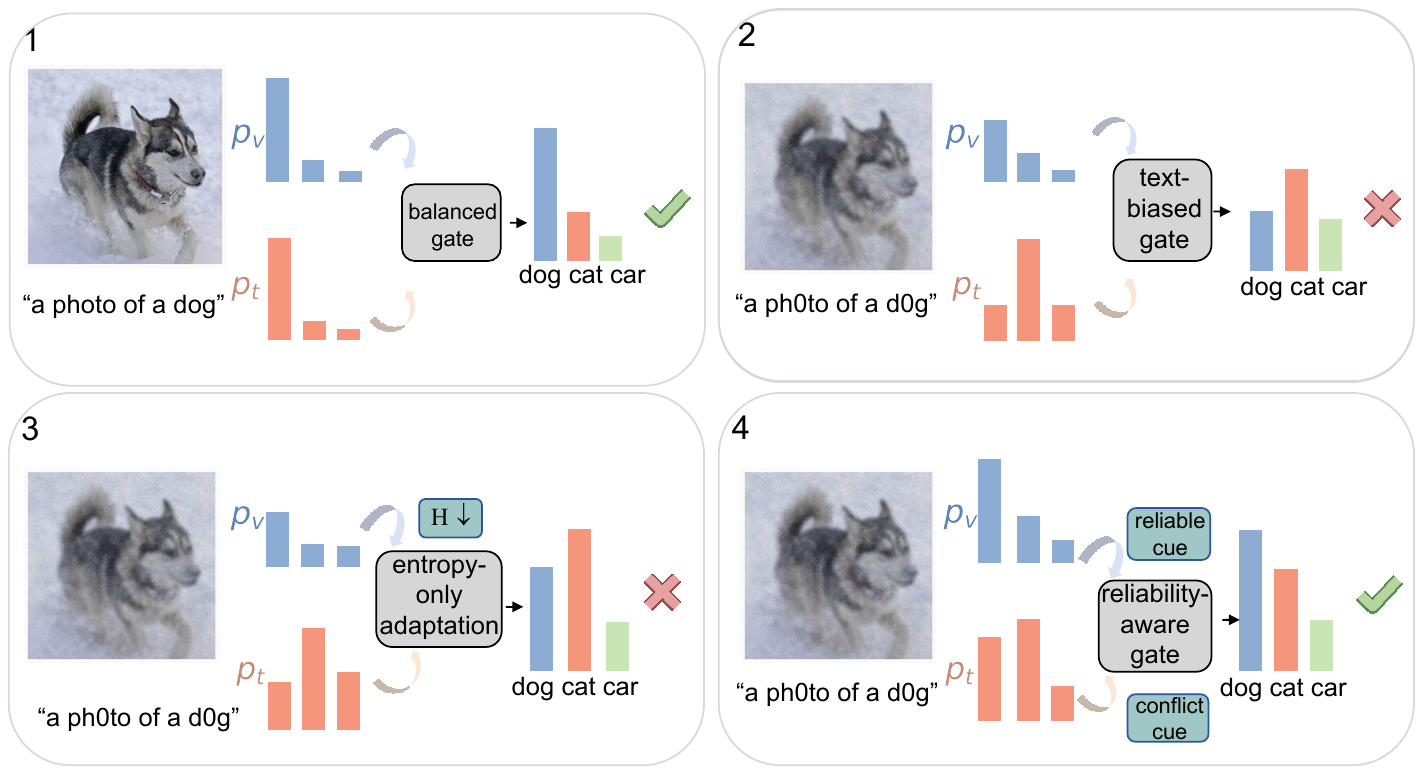}
    \caption{
    \textbf{Motivation of MG-MTTA.}
    Under modality-specific shift, entropy-driven adaptation can increase
    confidence without checking whether the dominant modality is
    reliable. MG-MTTA regulates modality contributions before posterior
    sharpening, reducing confident errors caused by unreliable modality
    dominance.
    }
    \label{fig:motivation}
\end{figure*}

When the branches shift differently, confidence ceases to be a common
scale of reliability. A degraded branch may still produce a concentrated
posterior and receive excessive fusion weight. Prompt-robustness and
multimodal-adaptation studies report related forms of this
\emph{reliability imbalance}~\cite{towards_robust_prompts,prompt_shift,
carot_vlm,reliability_bias}. The fusion rule should therefore account for relative branch reliability
when assigning modality weights.

Entropy minimization, a standard objective in test-time adaptation
~\cite{tent,eata,sar}, increases prediction confidence without identifying
which branch warrants that confidence. Under modality-specific shift,
fusion may become sharper because the unreliable modality is amplified.
The adapted prediction is then incorrect and more confident than its
pre-adaptation counterpart. We call this pattern a wrong-more-confident
(WMC) failure. As Fig.~\ref{fig:motivation} illustrates, fused confidence
alone cannot distinguish reliable agreement from biased modality dominance.
This highlights the difference between posterior concentration and
modality selection: entropy controls prediction sharpness, but not which
branch should dominate the fused decision.

MG-MTTA adapts the fusion path while keeping both encoders fixed. We model
modality-specific degradation as posterior mixing and use majorization to
characterize a valid de-mixing direction
~\cite{majorization_book,entropy_majorization_book}. Since clean
posteriors are unavailable at test time, running-anchor consistency tracks
relative branch drift, while Jensen--Shannon divergence and ranking
disagreement quantify cross-modal conflict~\cite{lin1991js}. These signals
define a prior over the fusion gate and constrain the modality direction
before entropy sharpening. Both encoders remain frozen; adaptation is
confined to the lightweight fusion module.

Our main contributions are as follows.
\begin{itemize}
    \item We identify a confidence--reliability mismatch in multimodal
    test-time adaptation: an unreliable modality can dominate fusion and
    turn entropy reduction into WMC failure.
    \item We formulate reliability-aware adaptation as constrained
    posterior de-mixing, deriving sufficient conditions for beneficial
    concentration and an explicit modality-dominance threshold.
    \item We implement the formulation as a lightweight fusion update and
    evaluate it through robustness results, reliability diagnosis, WMC
    analysis, and a strongest-probe objective ablation.
\end{itemize}

\section{Related Work}
\label{sec:related_work}

\subsection{VLM adaptation.}
CLIP and SigLIP learn transferable image--text representations from
large-scale paired data, and Flamingo extends visual--language pretraining
to few-shot multimodal generation and reasoning
~\cite{vlm_clip,vlm_siglip,vlm_flamingo}. More recently, Orca learns a
unified multimodal world latent space and evaluates it with a frozen
backbone and lightweight modality-specific readouts~\cite{wang2026orca}.
CoOp improves task adaptation by
replacing handcrafted context words with learnable prompt
vectors~\cite{zhou2022learning}. Later work studies distribution-aware
prompting, prompt consistency, cross-modal alignment, and robustness to
wording or domain changes
~\cite{prompt_shift,towards_robust_prompts,cpl_vlm,
craft_prompt_robust,carot_vlm}. At test time, TPT optimizes prompts from
augmented views, TDA uses a training-free dynamic cache, and comparative
studies examine how conventional online TTA transfers to
VLMs~\cite{tpt,tda,vte}. BATCLIP further adapts visual structure and
image--text alignment under common visual corruptions~\cite{batclip},
while other work considers unequal modality quality in general multimodal
models~\cite{reliability_bias}. These methods mainly adapt prompts,
representations, or alignment, whereas MG-MTTA keeps both VLM encoders
fixed and controls the fusion of branch predictions when confidence and
modality reliability disagree.

\subsection{Test-time adaptation.}
TTA updates a source model from unlabeled target data without source
retraining; a recent survey summarizes the main settings and design
choices~\cite{tta_survey}. TENT adapts normalization parameters by entropy
minimization, EATA filters unreliable samples and limits forgetting,
CoTTA stabilizes continual adaptation, and SAR combines reliable-sample
selection with sharpness-aware optimization
~\cite{tent,eata,cotta,sar}. These methods address confidence, stability,
and online drift along a single prediction path. Multimodal deployment
adds the need to compare the reliability of several paths before forming
the fused output.

\subsection{Reliability-aware multimodal fusion.}
Uncertainty compensation and adaptive weighting have long been used to
reduce the influence of noisy modalities
~\cite{uncertainty_compensation_fusion,multimodal_fusion}. Related
composed-retrieval studies also examine cross-modal inconsistency, modality
suppression under noisy correspondence, and directional bias in multimodal
composition~\cite{COMBINER,ConeSep,ReTrack}. Recent work studies provable
dynamic fusion for low-quality inputs, probabilistic late
fusion, and uncertainty-guided meta-learning
~\cite{provable_dynamic_fusion,reliability_fusion,
uncertainty_multimodal}. These methods assign modality weights according
to evidence quality. MG-MTTA follows this reliability-aware perspective
in source-free VLM adaptation and relates modality weighting to confidence
amplification under modality-specific shift.

\subsection{Majorization and entropy.}
Majorization orders probability vectors by concentration, with
Schur-concavity linking that order to entropy
~\cite{majorization_book,entropy_majorization_book}. Jensen--Shannon
(JS) divergence is a symmetric, bounded measure of distributional
disagreement~\cite{lin1991js}. In our analysis, majorization describes a
candidate de-mixing direction, whereas JS divergence provides an
observable conflict signal. Concentration and disagreement alone do not
certify correctness, so the theory also considers ranking preservation
and proximity to the clean margin.

\section{Method}
\label{sec:method}

\begin{figure*}[t]
\centering
\includegraphics[width=0.98\textwidth,trim={3.0cm 7.0cm 0 4.5cm},clip]
{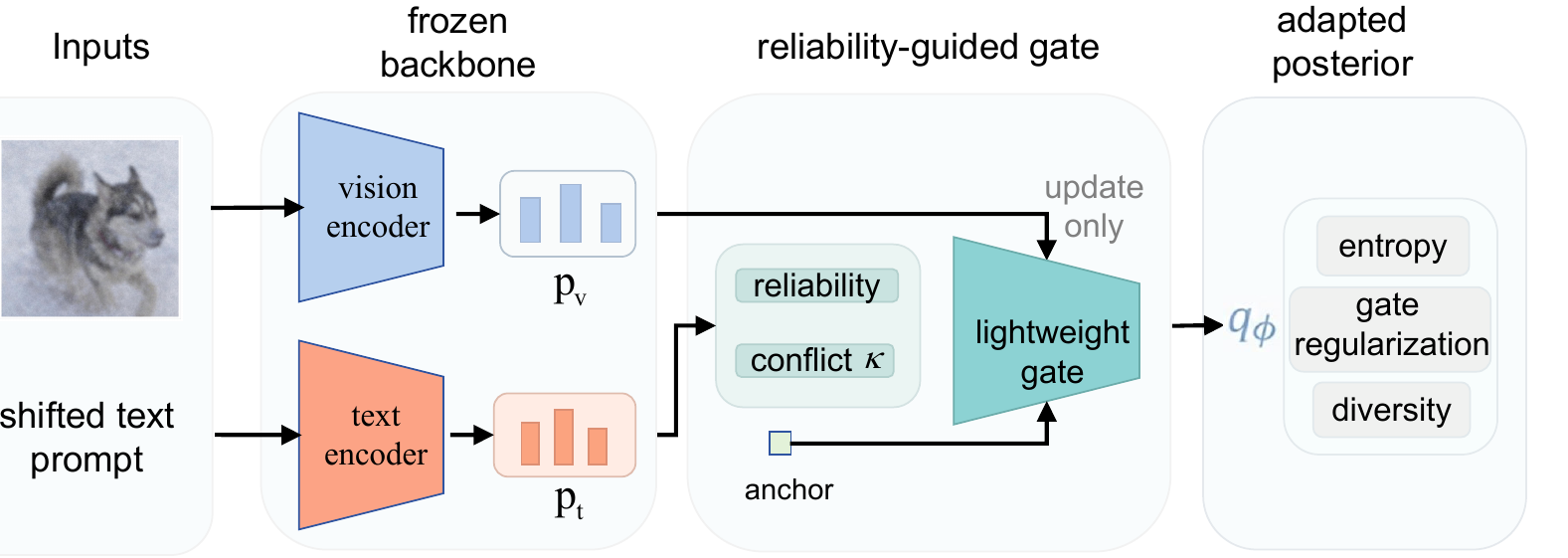}

\caption{
\textbf{Overview of MG-MTTA.}
The modality encoders remain frozen during adaptation.
MG-MTTA estimates modality reliability from prediction statistics
and uses reliability-aware fusion to regulate modality
contributions through an adaptive fusion gate.
}

\label{fig:framework}
\end{figure*}

\subsection{Reliability Mismatch as Posterior Control}
\label{sec:problem}

At test time, MG-MTTA controls modality contribution rather than encoder
representation. Entropy-based objectives specify how concentrated a
prediction becomes, but do not identify whether the branch driving that
concentration is reliable. Under modality-specific shift, an overconfident
branch can therefore dominate fusion and turn adaptation into a confident
error.

We formulate the problem as \textbf{reliability-aware posterior control}.
A majorization-based model describes posterior mixing, anchor drift and
cross-modal disagreement supply observable reliability cues, and a gate
prior uses these cues to regulate modality contribution during
entropy-based adaptation.

An overview of MG-MTTA is shown in Fig.~\ref{fig:framework}.
Given an unlabeled test sample $z=(x^v,x^t)$ with $K$ classes, the frozen
encoders produce modality logits
$h^v(z),h^t(z)\in\mathbb{R}^{K}$ and posteriors

\begin{equation}
p^m(z)=\operatorname{softmax}(h^m(z)),
\qquad m\in\{v,t\}.
\label{eq:modality_posteriors}
\end{equation}

The deployed fusion operates at the logit level:

\begin{align}
h_\phi^f(z)
&=
\alpha_\phi(z)h^v(z)
+
(1-\alpha_\phi(z))h^t(z),
\label{eq:fused_logit}\\
q_\phi(z)
&=
\operatorname{softmax}(h_\phi^f(z)).
\label{eq:fused_posterior}
\end{align}

where $\alpha_\phi(z)\in[0,1]$ is produced by the lightweight fusion
module parameterized by $\phi$. We denote its source initialization by
$\phi_0$ and write $p^f(z)=q_{\phi_0}(z)$ as the fused posterior before
adaptation.

For a test batch $\mathcal B$, the adaptation objective is

\begin{equation}
\mathcal L
=
\mathcal L_{\mathrm{ent}}
+
\lambda_g\mathcal L_{\mathrm{gate}}
+
\lambda_d\mathcal L_{\mathrm{div}}.
\label{eq:objective}
\end{equation}

Entropy promotes posterior concentration, the gate loss controls modality
preference, and diversity discourages collapse of the batch marginal.

The reliability terms constrain only the relative contribution of the
two branches and provide no class-level supervision; class selection
still follows the fused logits and entropy objective.

\paragraph{Scope of posterior control.}
A modality-specific shift need not erase all discriminative evidence from
the affected branch. Both posteriors may remain informative even after
their confidence values cease to be comparable measures of quality; the
error then lies in how the evidence is weighted. Updating an encoder may
improve its features, but does not directly constrain modality dominance.
MG-MTTA isolates this fusion error by treating the gate as the control
variable. It cannot reconstruct evidence lost by both encoders, but it
can prevent the remaining evidence from being weighted by confidence
alone. Posterior control is thus a lightweight complement to
representation-level adaptation.

\subsection{Posterior Mixing and Majorization}
\label{sec:mixing}

We first characterize how modality shift changes the branch posteriors.
The degradation is modeled as a mixing process:

\begin{equation}
p^m(z)
=
D^m(z)\pi^m(z)+r^m(z),
\qquad m\in\{v,t\},
\label{eq:mixing}
\end{equation}

where $D^m(z)$ is doubly stochastic and $r^m(z)$ captures deviations
from ideal mixing.

For posterior-space analysis, we define the latent clean fused target

\begin{equation}
\pi^f(z)
=
\beta^\star(z)\pi^v(z)
+
(1-\beta^\star(z))\pi^t(z),
\label{eq:clean_fusion}
\end{equation}

where $\beta^\star(z)\in[0,1]$ is the desired clean modality weight.
Eq.~\eqref{eq:clean_fusion} is used only for analysis; deployment
follows the logit-level rule in
Eqs.~\eqref{eq:fused_logit}--\eqref{eq:fused_posterior}. We do not equate
a convex mixture of clean posteriors with the softmax of fused logits.
The former is used for posterior-mixing analysis, whereas the latter
corresponds to the actual logit-level prediction path used at test time.

The posterior-space formulation is used for the majorization analysis,
whereas the deployed model performs logit-level fusion. Majorization is used to
reason about concentration in probability space, whereas the dominance
analysis in Sec.~\ref{sec:theory} returns to logit space, where changing
$\alpha_\phi$ can reverse the ordering of two classes.

For $u,w\in\Delta^{K-1}$, let their sorted coordinates be
$u_{[1]}\ge\cdots\ge u_{[K]}$ and
$w_{[1]}\ge\cdots\ge w_{[K]}$.
We say that $u$ majorizes $w$, written as $u\succeq w$, if

\begin{equation}
\sum_{i=1}^{k}u_{[i]}
\ge
\sum_{i=1}^{k}w_{[i]},
\quad k=1,\ldots,K-1,
\label{eq:majorization}
\end{equation}

with equality at $k=K$. Thus, $u$ is more concentrated, or less mixed,
than $w$.

By majorization theory, doubly stochastic transformations make a
distribution more mixed:

\begin{equation}
u\succeq Du
\Longrightarrow
H(u)\le H(Du),
\label{eq:entropy_majorization}
\end{equation}

where $H(u)=-\sum_k u_k\log u_k$.

Eq.~\eqref{eq:entropy_majorization} suggests de-mixing as a recovery
direction, but \textbf{lower entropy does not imply a correct
prediction}. A posterior may sharpen around the wrong class, so the
direction of concentration must also be controlled.

\subsection{Observable Reliability and Fusion Control}
\label{sec:reliability}

The clean posterior and mixing operator are unavailable at test time.
MG-MTTA replaces them with two observable signals: anchor consistency for
branch-specific drift and cross-modal conflict for incompatible modality
decisions.

\paragraph{Anchor-based reliability.}

An ideal residual would measure how well the observed posterior can be
represented as a doubly stochastic transformation of a clean anchor:

\begin{equation}
\rho_m^\star(z)
=
\min_{D\in\mathcal D_K}
\left\|
p^m(z)-D\bar\pi^m
\right\|_1,
\label{eq:ideal_residual}
\end{equation}

where $\mathcal D_K$ denotes the set of $K\times K$ doubly stochastic
matrices.

Since the clean anchor $\bar\pi^m$ and the optimal transformation are
unavailable at test time, we maintain a running anchor $\bar p^m$ and
use the following practical proxy:

\begin{equation}
\tilde\rho_m(z)
=
\left\|
\operatorname{sort}_{\downarrow}(p^m(z))
-
\operatorname{sort}_{\downarrow}(\bar p^m)
\right\|_1 .
\label{eq:proxy}
\end{equation}

A smaller $\tilde\rho_m$ indicates greater consistency with the modality
anchor. Sorting retains the concentration profile used by majorization
but removes class identity. Because sorting removes class identity, the proxy can detect profile
drift but cannot capture semantic class swaps on its own; cross-modal
conflict therefore accounts for disagreement in class predictions.
Accordingly, \textbf{$\tilde\rho_m$ is a relative reliability proxy, not
an estimator of the latent clean posterior.}

The running anchor serves as a local temporal reference rather than a
fixed clean-domain prototype. Consequently, the proxy measures departure
from recent stable behavior and may inherit errors from that reference.
Confidence filtering and EMA updating reduce abrupt contamination, but do
not turn the anchor into ground-truth supervision.

For modality $m$, confident samples and the anchor update are defined as

\begin{align}
\mathcal B_m^{\mathrm{conf}}
&=
\{z\in\mathcal B:
\max_k p_k^m(z)\ge\eta\},
\label{eq:confident_set}\\
\hat p^m
&=
\frac{1}{|\mathcal B_m^{\mathrm{conf}}|}
\sum_{z\in\mathcal B_m^{\mathrm{conf}}}
p^m(z),
\label{eq:batch_anchor}\\
\bar p^m
&\leftarrow
\mu\bar p^m+(1-\mu)\hat p^m.
\label{eq:anchor_update}
\end{align}

where $\eta$ is the confidence threshold and $\mu$ is the EMA momentum.
The anchor remains unchanged when
$\mathcal B_m^{\mathrm{conf}}=\varnothing$.

\paragraph{Cross-modal conflict.}

Anchor consistency does not show whether the modalities currently support
incompatible decisions. Let their midpoint posterior be

\begin{equation}
s(z)=\frac{1}{2}\left(p^v(z)+p^t(z)\right).
\label{eq:midpoint}
\end{equation}

The probability- and ranking-level conflict terms are

\begin{align}
\operatorname{JS}(p^v,p^t)
&=
\frac12\operatorname{KL}(p^v\|s)
+
\frac12\operatorname{KL}(p^t\|s),
\label{eq:js_conflict}\\
R(p^v,p^t)
&=
\frac{2N_{\mathrm{disc}}(p^v,p^t)}
{K(K-1)},
\label{eq:rank_conflict}\\
\kappa(z)
&=
\operatorname{JS}(p^v,p^t)
+
\lambda_rR(p^v,p^t).
\label{eq:combined_conflict}
\end{align}

where $N_{\mathrm{disc}}$ counts class pairs ordered differently by the
two modalities, and the normalization maps $R$ to $[0,1]$. JS measures
probability-level disagreement~\cite{lin1991js}; $R$ records ranking
conflict. The two are complementary: a small probability change may
reverse near-tied classes, whereas ranking disagreement ignores the size
of the probability gap.

Using the complete ranking also retains information beyond top-one
agreement. Two branches may predict the same class while ordering the
remaining alternatives differently, which can matter once the gate moves
or the leading classes are close. Conversely, a rank reversal with a
negligible probability gap should not carry the same weight as a large
probability disagreement.

Entropy, anchor deviation, and conflict describe different parts of the
adaptation state. Entropy measures concentration but can be low for an
incorrect class. Anchor deviation tracks departure from a recent branch
profile, while conflict shows whether the two branches support
incompatible decisions. A drifting branch may still agree with the other
branch, and disagreement may occur even when their deviations are
similar. Relative anchor deviation therefore sets the modality
preference; conflict strengthens that preference only when dominance can
change the fused decision.

\paragraph{Reliability-aware gate prior.}

The reliability statistics determine a modality preference but do not
alter the prediction by themselves. We encode that preference as a gate
prior that constrains the direction of entropy-based adaptation.

\begin{align}
d(z)
&=
\operatorname{sign}
\left(
\tilde\rho_v(z)-\tilde\rho_t(z)
\right),
\label{eq:reliability_direction}\\
\begin{bmatrix}
\ell^v(z)\\
\ell^t(z)
\end{bmatrix}
&=
-\tau
\begin{bmatrix}
\tilde\rho_v(z)\\
\tilde\rho_t(z)
\end{bmatrix}
+
\lambda_c\kappa(z)d(z)
\begin{bmatrix}
-1\\
1
\end{bmatrix},
\label{eq:gate_logits}\\
[a^v(z),a^t(z)]
&=
\operatorname{softmax}
\left(
\ell^v(z),\ell^t(z)
\right).
\label{eq:gate_prior}
\end{align}

Here $\tau>0$ controls reliability sensitivity and
$\lambda_c\ge0$ controls conflict correction. The correction is
asymmetric because adding the same offset to both logits would cancel
after softmax.

The sign in Eq.~\eqref{eq:reliability_direction} gives the correction an
explicit direction. If $\tilde\rho_v>\tilde\rho_t$, the visual branch has
the larger deviation, so the conflict term lowers $\ell^v$ and raises
$\ell^t$; the opposite occurs when the textual deviation is larger.
Conflict therefore changes the strength of an already identified
preference rather than selecting a modality independently.

The objective terms in Eq.~\eqref{eq:objective} are

\begin{align}
\mathcal L_{\mathrm{ent}}
&=
\frac{1}{|\mathcal B|}
\sum_{z\in\mathcal B}H(q_\phi(z)),
\label{eq:entropy_loss}\\
\mathcal L_{\mathrm{gate}}
&=
\frac{1}{|\mathcal B|}
\sum_{z\in\mathcal B}
\operatorname{KL}
\left(
[\alpha_\phi(z),1-\alpha_\phi(z)]
\|
[a^v(z),a^t(z)]
\right),
\label{eq:gate_loss}\\
\mathcal L_{\mathrm{div}}
&=
-H
\left(
\frac{1}{|\mathcal B|}
\sum_{z\in\mathcal B}q_\phi(z)
\right).
\label{eq:diversity_loss}
\end{align}

The objective separates modality selection from confidence improvement:
the gate loss controls dominance, entropy sharpens the posterior in the
selected direction, and diversity regularizes the batch marginal.

We use the gate prior as a soft constraint. The KL term penalizes a fusion weight that conflicts with the reliability
statistics, while still allowing the learned gate to deviate from the
prior when supported by the prediction objective. This allows flexible adaptation while preventing entropy alone from
determining modality dominance.

\paragraph{Online update.}
For each batch, the frozen encoders produce modality logits and
posteriors. Confidence-filtered statistics update the running anchors by
Eq.~\eqref{eq:anchor_update}; a branch retains its previous anchor when no
sample passes the threshold. These deviations and cross-modal conflict
define the gate prior in Eq.~\eqref{eq:gate_prior}, and gradients
from Eq.~\eqref{eq:objective} update only $\phi$. The adaptive state is
therefore limited to the fusion module and two $K$-dimensional anchors.
The statistics determine how modality contribution should move, while
entropy determines how strongly the posterior sharpens along that
controlled direction. We next analyze sufficient conditions under which this controlled
sharpening preserves the clean decision.

\subsection{Theoretical Characterization}
\label{sec:theory}

We now state conditions under which posterior control reduces entropy
without changing the clean decision.

Let

\begin{align}
c^\star
&=
\argmax_k\pi_k^f(z),
\label{eq:clean_top_class}\\
j^\star
&=
\argmax_{j\neq c^\star}\pi_j^f(z),
\label{eq:clean_runner_up}\\
\gamma(z)
&=
\pi_{c^\star}^f(z)
-
\pi_{j^\star}^f(z).
\label{eq:clean_margin}
\end{align}

where $\gamma(z)>0$ denotes the clean top-one margin.

\begin{theorem}[Beneficial de-mixing]
\label{thm:demixing}

Suppose that the adapted posterior satisfies the following three
conditions:

\begin{align}
q_\phi(z)
&\succeq p^f(z),
\label{eq:condition_majorization}\\
q_{\phi,c^\star}(z)
&>q_{\phi,j^\star}(z),
\label{eq:condition_pair_order}\\
\left\|q_\phi(z)-\pi^f(z)\right\|_\infty
&<\frac{\gamma(z)}{2}.
\label{eq:condition_proximity}
\end{align}

Then

\begin{align}
H(q_\phi(z))
&\le H(p^f(z)),
\label{eq:beneficial_entropy}\\
\argmax_k q_{\phi,k}(z)
&=c^\star.
\label{eq:beneficial_prediction}
\end{align}

\end{theorem}

\noindent\textit{Proof sketch.}
Eq.~\eqref{eq:condition_majorization} and Schur-concavity give
Eq.~\eqref{eq:beneficial_entropy}. For every $j\neq c^\star$,

\begin{equation}
q_{\phi,c^\star}-q_{\phi,j}
\ge
\pi_{c^\star}^f-\pi_j^f
-
2\|q_\phi-\pi^f\|_\infty
>0,
\label{eq:proof_margin_bound}
\end{equation}

which proves Eq.~\eqref{eq:beneficial_prediction}. The pair-order
condition explicitly records preservation of the decisive clean pair.

The theorem makes explicit three properties that entropy minimization
alone does not guarantee: a valid concentration direction, preservation
of the decisive ranking, and proximity to the clean fused posterior. These quantities are
latent at deployment; anchor and conflict are therefore
\emph{theory-aligned control signals}, not tests of the sufficient
conditions.

In particular, majorization alone establishes only an entropy ordering.
Correctness follows only after the decisive ranking and clean-margin
conditions are included, which is why a lower-entropy adapted posterior
cannot be treated as evidence that de-mixing succeeded.

\paragraph{Interpretation of the sufficient conditions.}
Each condition excludes a different failure mode.
Eq.~\eqref{eq:condition_majorization} requires movement toward a less
mixed posterior but still permits sharpening the wrong class;
Eq.~\eqref{eq:condition_pair_order} prevents reversal of the decisive
clean pair; and Eq.~\eqref{eq:condition_proximity} defines a
margin-preserving neighborhood around the clean fusion. Since
$c^\star$, $\gamma(z)$, and $\pi^f(z)$ are unobserved, MG-MTTA cannot
certify these conditions per sample. Anchor deviation instead warns of
branch drift, and cross-modal conflict marks cases in which a wrong
modality preference can alter the fused ranking. Since the required clean quantities are unavailable at test time, these
conditions cannot be verified on individual test samples.

\paragraph{Modality-dominance threshold.}

Because the deployed rule fuses modality logits, we characterize
dominance directly in logit space. For classes $c$ and $j$, define

\begin{align}
\delta^v_{c,j}
&=
h_c^v-h_j^v,
\label{eq:visual_logit_margin}\\
\delta^t_{c,j}
&=
h_c^t-h_j^t.
\label{eq:text_logit_margin}
\end{align}

\begin{proposition}[Modality-dominance threshold]
\label{prop:dominance}

Assume
$\delta^v_{c,j}<0<\delta^t_{c,j}$,
so that the visual branch favors $j$ while the textual branch favors
$c$. If

\begin{equation}
\alpha_\phi(z)
>
\frac{
\delta^t_{c,j}(z)
}{
\delta^t_{c,j}(z)-\delta^v_{c,j}(z)
},
\label{eq:dominance_threshold}
\end{equation}

then

\begin{equation}
q_{\phi,c}(z)<q_{\phi,j}(z).
\label{eq:ranking_reversal}
\end{equation}

\end{proposition}

\noindent\textit{Proof sketch.}
The fused logit margin is

\begin{equation}
h_{\phi,c}^f-h_{\phi,j}^f
=
\alpha_\phi\delta^v_{c,j}
+
(1-\alpha_\phi)\delta^t_{c,j}.
\label{eq:fused_logit_margin}
\end{equation}

Eq.~\eqref{eq:dominance_threshold} makes
Eq.~\eqref{eq:fused_logit_margin} negative, and softmax preserves
coordinate ordering, yielding Eq.~\eqref{eq:ranking_reversal}.

The threshold depends on the competing visual and textual margins rather
than on a universal gate value. A strongly biased branch can reverse the
fusion with a smaller weight, whereas a reliable branch with a large
opposing margin tolerates more contribution from the corrupted modality.

Proposition~\ref{prop:dominance} connects the failure to the gate: once
an unreliable modality exceeds the threshold, the fused ranking reverses.
MG-MTTA constrains this direction explicitly rather than leaving modality
dominance to entropy optimization. We next evaluate whether the proposed reliability signals guide fusion
effectively and reduce confidence-amplified errors in practice.

\section{Experiments}
\label{sec:experiments}

We evaluate both robustness and the proposed failure mechanism. The main
benchmark covers visual, textual, and joint shifts; separate diagnostics
relate the proxy to branch quality, measure confidence-amplified errors,
and compare objective variants under severe degradation.

We organize the evaluation around three questions: whether fusion
control improves robustness, whether the proxy ranks branch states by
relative quality, and whether matched entropy/diversity objectives
reproduce the same recovery under severe text-side degradation. The main
benchmark, mechanism diagnostics, and the strongest-probe objective
ablation are reported separately because they support different parts of
the argument.

\subsection{Experimental Setup}
\label{sec:protocol}

\paragraph{Shift settings.}
We evaluate visual, textual, and joint modality-specific shifts. Visual
shift uses CIFAR-100C and ImageNet-C with the 15 common corruption types
at severity~5~\cite{imagenet_c}. For textual shift, the image remains
unchanged while the prompt bank undergoes four progressively stronger
label-preserving perturbation levels (L1--L4), covering changes in
context, lexical form, tokenization, and local surface form. The protocol
keeps class identity fixed so that changes in the fused prediction can be
attributed to text-side degradation rather than semantic label changes.
It captures controlled forms of deployment-time variation in prompts,
templates, metadata, and tokenization, but is not intended as an
exhaustive benchmark for natural paraphrases or unrestricted user
queries. Textual accuracy is averaged over L1--L4. Joint shift combines
each textual condition with severity-5 visual corruption and reports the
corresponding average. L5 is reserved as a separate strongest text-side
probe and is excluded from the main benchmark.

\paragraph{Diagnostic settings.}
Two diagnostic protocols examine the reliability mechanism. A
severity-1/3/5 protocol tests whether anchor deviation responds to
branch-specific stress and separates branch states of different quality.
A second stream measures wrong-more-confident (WMC) events and total
fused-prediction error under textual and joint shift. A WMC event is an
incorrect adapted prediction whose confidence exceeds that of the
pre-adaptation fused prediction. Both diagnostics, as well as the L5
probe, are separate from the L1--L4 benchmark.

\paragraph{Baselines and implementation.}
Source-only performs no adaptation. Entropy-only updates the same
lightweight parameter subset as MG-MTTA by minimizing fused prediction
entropy, while Entropy + Diversity additionally optimizes the
batch-marginal diversity term. These matched baselines test whether
confidence sharpening and output diversity are sufficient without
differences in trainable capacity. External comparisons include
TENT~\cite{tent}, SAR~\cite{sar}, TPT~\cite{tpt}, and
BATCLIP~\cite{batclip}. TENT is reported as a visual-shift reference,
whereas SAR, TPT, and BATCLIP are evaluated under the same textual/joint
inputs and target ordering as MG-MTTA while retaining their original
adaptation procedures. The primary backbone is CLIP
ViT-B/16~\cite{vlm_clip}. MG-MTTA uses a batch size of 64, two adaptation
steps per batch, a learning rate of $10^{-3}$, confidence threshold
$\eta=0.7$, and EMA momentum $\mu=0.9$. The model is reset before each
shift condition. Experiments were conducted on a single NVIDIA H20 GPU,
and results are averaged over three random seeds.

\begin{table*}[t]
\centering
\caption{
\textbf{Performance and adaptation cost under modality-specific shift.}
Accuracy is reported in percentage. Textual results are averaged over
L1--L4, while Joint combines severity-5 visual corruption with L1--L4
textual shift. Runtime denotes the complete textual/joint evaluation.
TENT is reported only for visual shift. Bold and underlined accuracy values indicate
the best and second-best results, respectively. Cost minima are marked
among adapted methods with available measurements, and dashes denote
unavailable entries.
}
\label{tab:main_results}
\small
\setlength{\tabcolsep}{3.5pt}
\renewcommand{\arraystretch}{1.12}
\begin{tabular}{@{}lccccccc@{}}
\toprule
\multirow{2}{*}{Method}
& \multicolumn{2}{c}{Visual shift}
& \multicolumn{2}{c}{Multimodal shift}
& \multirow{2}{*}{\makecell{Trainable\\params.}}
& \multirow{2}{*}{\makecell{Runtime\\Text/Joint}}
& \multirow{2}{*}{\makecell{Peak\\memory}} \\
\cmidrule(lr){2-3}
\cmidrule(lr){4-5}
& CIFAR-100C
& ImageNet-C
& Textual
& Joint
& & & \\
\midrule

\multicolumn{8}{@{}l}{\textit{Frozen source model}} \\
Source-only
& 35.80
& 25.45
& 57.97
& 21.68
& 0
& 01:29 / 02:14
& -- \\
\addlinespace[2pt]

\multicolumn{8}{@{}l}{\textit{External adaptation methods}} \\
TENT
& 37.96
& 25.64
& --
& --
& --
& --
& -- \\

SAR
& \underline{41.42}
& \underline{29.70}
& 56.84
& 25.61
& 31.7K
& 36:42 / 56:37
& 1.95 GB \\

TPT
& 36.20
& 24.90
& 56.05
& 24.07
& \textbf{2.0K}
& 02:41:00 / 04:04:55
& 5.80 GB \\

BATCLIP
& \textbf{42.21}
& \textbf{30.70}
& \underline{60.61}
& \textbf{29.29}
& 65.5K
& 40:57 / 01:05:38
& 2.60 GB \\
\addlinespace[2pt]

\multicolumn{8}{@{}l}{\textit{Matched adaptation baselines}} \\
Entropy-only
& 37.96
& 26.39
& 55.40
& 21.68
& 41.0K
& --
& -- \\

Entropy + Diversity
& 38.10
& 26.29
& 55.38
& 21.76
& 41.0K
& --
& -- \\
\midrule

\textbf{MG-MTTA (Ours)}
& 38.11
& 26.32
& \textbf{66.51}
& \underline{26.27}
& 41.0K
& \textbf{26:26 / 39:32}
& \textbf{1.64 GB} \\
\bottomrule
\end{tabular}
\end{table*}

\subsection{Main Results}
\label{sec:overall}

Table~\ref{tab:main_results} reports top-1 accuracy and adaptation cost.
Visual results are averaged over 15 severity-5 corruptions, whereas
Textual and Joint results are averaged over L1--L4.

The matched baselines show that confidence alone can be misleading.
Under textual shift, Entropy-only lowers accuracy from 57.97\% to
55.40\%, and adding diversity gives 55.38\%. Under joint shift,
Entropy-only remains at the 21.68\% Source-only result, while Entropy +
Diversity reaches 21.76\%. Lower fused entropy therefore does not identify
a reliable adaptation direction. If the more corrupted modality already
dominates fusion, sharpening can reinforce rather than correct that
dominance.

MG-MTTA reaches 66.51\% under textual shift, 8.54 points above
Source-only and 11.11 above Entropy-only. Under joint shift, MG-MTTA reaches 26.27\%, improving over
Entropy-only and Entropy + Diversity by 4.59 and 4.51 points,
respectively. The gains are largest
when branch quality is asymmetric and modality weighting directly affects
the fused decision. These results indicate that entropy reduction alone is insufficient to
determine which modality should dominate the fused prediction.

Under textual shift, the unchanged visual branch provides a clearer
relative reference for the gate. Under joint shift, both references may
deteriorate, making the preferred fusion direction less distinct. The larger gain under textual shift is consistent with this asymmetry,
although the proxy remains a relative ranking signal rather than a
clean-modality estimator.

On visual-only shifts, MG-MTTA obtains 38.11\% on CIFAR-100C and
26.32\% on ImageNet-C, below SAR and BATCLIP; BATCLIP also leads on
the joint average. The methods intervene at
different levels: BATCLIP adapts bimodal representations and alignment,
whereas MG-MTTA regulates modality contribution after branch predictions
are formed. These results indicate that fusion-level adaptation complements, rather
than replaces, alignment-based adaptation.

MG-MTTA updates 41.0K parameters, completes the textual and joint
evaluations in 26:26 and 39:32, and uses 1.64\,GB peak memory. Since both
encoders are frozen, this cost comes only from the lightweight fusion
parameters. Thus, the additional adaptation cost is limited to the lightweight fusion
module. Aggregate accuracy, however, does not show whether the gate
responds to the intended modality. We therefore examine the reliability
proxy directly.

\subsection{Reliability Analysis}
\label{sec:reliability_diagnostic}

As textual severity increases from 1 to 3 and 5, the text-side anchor
deviation rises from 1.154 to 1.213 and 1.617, while the visual-side
reference remains near 1.138, indicating a branch-selective response.
In the same diagnostic stream, MG-MTTA improves over the corresponding
pre-adaptation predictions by 0.50, 2.90, and 0.55 percentage points
under textual shift and by 1.50, 3.40, and 3.70 points under joint shift
at severities 1, 3, and 5, respectively. These diagnostics are separate from the L1--L4 benchmark. They support the
proxy as a relative ranking signal, not as a calibrated estimator or a
per-sample test of the theoretical conditions.

\begin{figure*}[t]
    \centering
    \includegraphics[
        width=0.88\textwidth
    ]{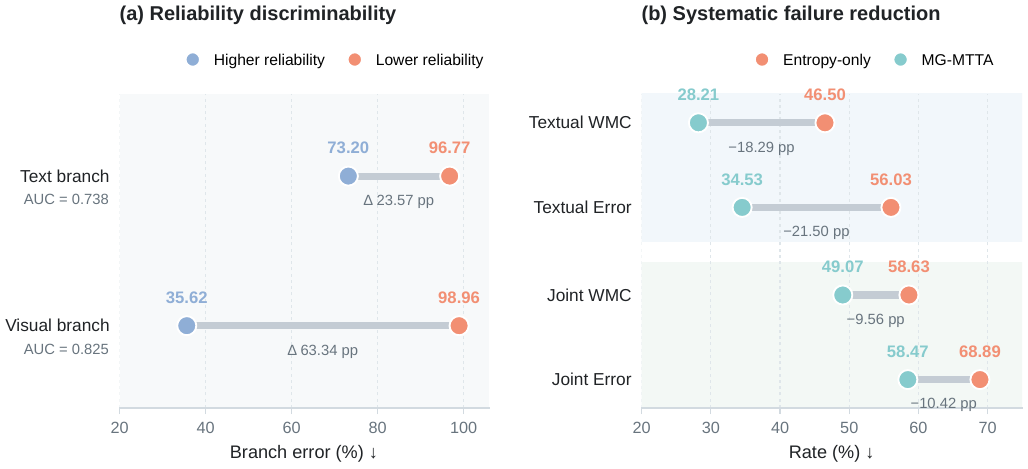}
    \caption{
    \textbf{Reliability and systematic failure diagnostics.}
    \textbf{(a)} Branches assigned higher relative reliability exhibit
    lower branch-specific prediction error, with AUCs of 0.738 and 0.825
    for the textual and visual branches, respectively.
    \textbf{(b)} MG-MTTA reduces wrong-more-confident (WMC) events and
    total fused-prediction error relative to Entropy-only under textual
    and joint shift. Panel~(a) reports branch-level statistics, whereas
    Panel~(b) evaluates the final fused prediction. The panels use
    separate diagnostic protocols and are not derived from the L1--L4
    benchmark.
    }
    \label{fig:reliability_failure}
\end{figure*}

Figure~\ref{fig:reliability_failure}(a) groups branch states by relative
reliability. For the textual branch, the higher- and lower-reliability
groups have error rates of 73.20\% and 96.77\%, respectively, with
AUC $=0.738$. The corresponding visual-branch errors are 35.62\% and
98.96\%, with AUC $=0.825$. These branch-level statistics assess ranking
quality rather than calibration or final fused accuracy. Although higher
relative reliability is associated with lower error, the 73.20\% error
of the higher-reliability textual group shows that the proxy is not an
oracle of correctness. It supplies only a directional signal for reducing the influence of the
branch with greater anchor deviation.

\subsection{Failure Analysis}
\label{sec:wmc_diagnostic}
To further analyze failure behavior beyond top-1 accuracy, we measure
wrong-more-confident (WMC) events. A WMC event occurs when adaptation
produces an incorrect fused prediction with higher confidence than its
pre-adaptation counterpart.

Figure~\ref{fig:reliability_failure}(b) reports WMC and total fused error.
Under textual shift, Entropy-only yields 46.50\% WMC and 56.03\% error;
MG-MTTA lowers them to 28.21\% and 34.53\%, reductions of 18.29 and
21.50 points. The accuracy gain is accompanied by fewer cases in which adaptation
increases confidence in an incorrect prediction; the reduction is not
explained solely by more correct final labels.

Under joint shift, Entropy-only yields 58.63\% WMC and 68.89\% total
error, while MG-MTTA reduces them to 49.07\% and 58.47\%, drops of 9.56
and 10.42 points. The remaining WMC rate is higher than under textual
shift, consistent with both modality references being degraded. Even in
this case, controlling modality contribution reduces both overall error
and confidence amplification.

Entropy-only controls posterior sharpness but not which modality drives
the fused decision. MG-MTTA constrains this direction through the gate
before sharpening, so the comparison isolates modality control rather
than trainable capacity. The observed WMC reduction supports the proposed
mechanism but remains empirical and does not verify the theoretical
sufficient conditions.

\subsection{Ablation Study}
\label{sec:ablation}

The L5 condition tests whether entropy minimization or batch diversity
can explain recovery under the strongest textual perturbation. The
matched baselines use the same trainable parameters as MG-MTTA but omit
reliability-aware fusion control.

\begin{table}[t]
\centering
\caption{
\textbf{Objective ablation under the L5 strongest probe.}
Results are top-1 accuracy (\%). L5 is a separate diagnostic condition
and is excluded from the L1--L4 main benchmark.
}
\label{tab:l5_ablation}
\small
\setlength{\tabcolsep}{6.5pt}
\renewcommand{\arraystretch}{1.10}
\begin{tabular}{@{}lcc@{}}
\toprule
Method
& Textual L5
& Joint L5 \\
\midrule
Source-only
& 25.48
& 9.66 \\

Entropy-only
& 22.11
& 9.40 \\

Entropy + Diversity
& 22.14
& 9.39 \\

\textbf{MG-MTTA (Ours)}
& \textbf{65.88}
& \textbf{26.13} \\
\bottomrule
\end{tabular}
\end{table}

Table~\ref{tab:l5_ablation} shows that confidence sharpening does not
recover from severe textual degradation. Entropy-only lowers textual L5
accuracy from 25.48\% to 22.11\% and joint L5 accuracy from 9.66\% to
9.40\%; adding diversity gives 22.14\% and 9.39\%. Batch diversity can
prevent global output collapse, but cannot decide which branch should
dominate an individual prediction. It is insufficient when the failure
is reliability imbalance rather than inadequate output diversity.

MG-MTTA reaches 65.88\% under textual L5 and 26.13\% under joint L5,
40.40 and 16.47 points above Source-only. The matched baselines use the same trainable capacity, so parameter count
does not explain the recovery. The main difference from the matched baselines is therefore the
reliability-guided gate constraint.

The ablation does not isolate the anchor, conflict, and gate signals
individually, nor does it establish necessity for every adaptation
framework. Its conclusion is limited to the matched setting: entropy and
diversity do not account for the recovery without explicit control of
modality contribution.

Overall, the diagnostics show that the reliability proxy responds to
modality-specific degradation and that MG-MTTA reduces WMC failures.
The L5 ablation further shows that matched entropy and diversity
objectives do not explain the observed recovery.

\section{Conclusion and Limitations}

MG-MTTA treats multimodal TTA as control of modality contribution, not
confidence alone. Under modality-specific shift, entropy minimization can
amplify an unreliable branch and produce WMC errors. Majorization
characterizes a candidate de-mixing direction, while anchor consistency
and cross-modal conflict guide a lightweight fusion gate with frozen
encoders.

Under controlled textual and joint shifts, reliability-aware fusion
improves accuracy and reduces WMC relative to matched entropy-based
baselines. Visual-only results remain below SAR and BATCLIP, so fusion
control complements rather than replaces representation- and
alignment-level adaptation.

The evaluation is limited to classification under controlled,
label-preserving visual and textual perturbations; retrieval, VQA,
unrestricted user queries, and broad paraphrase distributions are outside
scope. The proxy may weaken when both modalities degrade similarly or
anchors accumulate bias, and fusion control cannot recover evidence lost
inside an encoder. The current formulation assumes two-modality late
fusion; target-order and batch-size sensitivity, stronger temporal
references, and broader tasks remain open.

\begin{acks}
The authors gratefully acknowledge partial support from the
Undergraduate Innovation Training Program of Guangdong University of
Technology (Project No.~xj2026118450434).
\end{acks}

\bibliographystyle{ACM-Reference-Format}
\bibliography{ref}

@article{tta_survey,
  title = {A Comprehensive Survey on Test-Time Adaptation Under Distribution Shifts},
  volume = {133},
  ISSN = {1573-1405},
  url = {http://dx.doi.org/10.1007/s11263-024-02181-w},
  DOI = {10.1007/s11263-024-02181-w},
  number = {1},
  journal = {International Journal of Computer Vision},
  publisher = {Springer Science and Business Media LLC},
  author = {Liang,  Jian and He,  Ran and Tan,  Tieniu},
  year = {2024},
  month = jul,
  pages = {31–64}
}

@misc{tent,
  doi = {10.48550/ARXIV.2006.10726},
  url = {https://arxiv.org/abs/2006.10726},
  author = {Wang,  Dequan and Shelhamer,  Evan and Liu,  Shaoteng and Olshausen,  Bruno and Darrell,  Trevor},
  title = {Tent: Fully Test-time Adaptation by Entropy Minimization},
  publisher = {arXiv},
  year = {2020},
  copyright = {arXiv.org perpetual,  non-exclusive license}
}

@misc{cotta,
  doi = {10.48550/ARXIV.2203.13591},
  url = {https://arxiv.org/abs/2203.13591},
  author = {Wang,  Qin and Fink,  Olga and Van Gool,  Luc and Dai,  Dengxin},
  title = {Continual Test-Time Domain Adaptation},
  publisher = {arXiv},
  year = {2022},
  copyright = {arXiv.org perpetual,  non-exclusive license}
}

@misc{eata,
  doi = {10.48550/ARXIV.2204.02610},
  url = {https://arxiv.org/abs/2204.02610},
  author = {Niu,  Shuaicheng and Wu,  Jiaxiang and Zhang,  Yifan and Chen,  Yaofo and Zheng,  Shijian and Zhao,  Peilin and Tan,  Mingkui},
  title = {Efficient Test-Time Model Adaptation without Forgetting},
  publisher = {arXiv},
  year = {2022},
  copyright = {arXiv.org perpetual,  non-exclusive license}
}

@misc{sar,
  doi = {10.48550/ARXIV.2302.12400},
  url = {https://arxiv.org/abs/2302.12400},
  author = {Niu,  Shuaicheng and Wu,  Jiaxiang and Zhang,  Yifan and Wen,  Zhiquan and Chen,  Yaofo and Zhao,  Peilin and Tan,  Mingkui},
  title = {Towards Stable Test-Time Adaptation in Dynamic Wild World},
  publisher = {arXiv},
  year = {2023},
  copyright = {arXiv.org perpetual,  non-exclusive license}
}

@misc{vlm_clip,
  doi = {10.48550/ARXIV.2103.00020},
  url = {https://arxiv.org/abs/2103.00020},
  author = {Radford,  Alec and Kim,  Jong Wook and Hallacy,  Chris and Ramesh,  Aditya and Goh,  Gabriel and Agarwal,  Sandhini and Sastry,  Girish and Askell,  Amanda and Mishkin,  Pamela and Clark,  Jack and Krueger,  Gretchen and Sutskever,  Ilya},
  title = {Learning Transferable Visual Models From Natural Language Supervision},
  publisher = {arXiv},
  year = {2021},
  copyright = {arXiv.org perpetual,  non-exclusive license}
}

@misc{vlm_siglip,
  doi = {10.48550/ARXIV.2303.15343},
  url = {https://arxiv.org/abs/2303.15343},
  author = {Zhai,  Xiaohua and Mustafa,  Basil and Kolesnikov,  Alexander and Beyer,  Lucas},
  title = {Sigmoid Loss for Language Image Pre-Training},
  publisher = {arXiv},
  year = {2023},
  copyright = {Creative Commons Attribution 4.0 International}
}

@misc{vlm_flamingo,
  doi = {10.48550/ARXIV.2204.14198},
  url = {https://arxiv.org/abs/2204.14198},
  author = {Alayrac,  Jean-Baptiste and Donahue,  Jeff and Luc,  Pauline and Miech,  Antoine and Barr,  Iain and Hasson,  Yana and Lenc,  Karel and Mensch,  Arthur and Millican,  Katie and Reynolds,  Malcolm and Ring,  Roman and Rutherford,  Eliza and Cabi,  Serkan and Han,  Tengda and Gong,  Zhitao and Samangooei,  Sina and Monteiro,  Marianne and Menick,  Jacob and Borgeaud,  Sebastian and Brock,  Andrew and Nematzadeh,  Aida and Sharifzadeh,  Sahand and Binkowski,  Mikolaj and Barreira,  Ricardo and Vinyals,  Oriol and Zisserman,  Andrew and Simonyan,  Karen},
  title = {Flamingo: a Visual Language Model for Few-Shot Learning},
  publisher = {arXiv},
  year = {2022},
  copyright = {arXiv.org perpetual,  non-exclusive license}
}

@misc{imagenet_c,
  doi = {10.48550/ARXIV.1903.12261},
  url = {https://arxiv.org/abs/1903.12261},
  author = {Hendrycks,  Dan and Dietterich,  Thomas},
  title = {Benchmarking Neural Network Robustness to Common Corruptions and Perturbations},
  publisher = {arXiv},
  year = {2019},
  copyright = {arXiv.org perpetual,  non-exclusive license}
}

@misc{prompt_shift,
  doi = {10.48550/ARXIV.2402.10099},
  url = {https://arxiv.org/abs/2402.10099},
  author = {Xiao,  Zehao and Shen,  Jiayi and Derakhshani,  Mohammad Mahdi and Liao,  Shengcai and Snoek,  Cees G. M.},
  title = {Any-Shift Prompting for Generalization over Distributions},
  publisher = {arXiv},
  year = {2024},
  copyright = {Creative Commons Attribution 4.0 International}
}

@misc{tpt,
  doi = {10.48550/ARXIV.2209.07511},
  url = {https://arxiv.org/abs/2209.07511},
  author = {Shu,  Manli and Nie,  Weili and Huang,  De-An and Yu,  Zhiding and Goldstein,  Tom and Anandkumar,  Anima and Xiao,  Chaowei},
  title = {Test-Time Prompt Tuning for Zero-Shot Generalization in Vision-Language Models},
  publisher = {arXiv},
  year = {2022},
  copyright = {arXiv.org perpetual,  non-exclusive license}
}

@misc{tda,
  doi = {10.48550/ARXIV.2403.18293},
  url = {https://arxiv.org/abs/2403.18293},
  author = {Karmanov,  Adilbek and Guan,  Dayan and Lu,  Shijian and Saddik,  Abdulmotaleb El and Xing,  Eric},
  title = {Efficient Test-Time Adaptation of Vision-Language Models},
  publisher = {arXiv},
  year = {2024},
  copyright = {Creative Commons Attribution Non Commercial No Derivatives 4.0 International}
}

@misc{batclip,
  doi = {10.48550/ARXIV.2412.02837},
  url = {https://arxiv.org/abs/2412.02837},
  author = {Maharana,  Sarthak Kumar and Zhang,  Baoming and Karlinsky,  Leonid and Feris,  Rogerio and Guo,  Yunhui},
  title = {$\texttt{BATCLIP}$: Bimodal Online Test-Time Adaptation for CLIP},
  publisher = {arXiv},
  year = {2024},
  copyright = {Creative Commons Attribution 4.0 International}
}

@inproceedings{
reliability_bias,
title={Test-time Adaptation against Multi-modal Reliability Bias},
author={Mouxing Yang and Yunfan Li and Changqing Zhang and Peng Hu and Xi Peng},
booktitle={The Twelfth International Conference on Learning Representations},
year={2024},
url={https://openreview.net/forum?id=TPZRq4FALB}
}

@book{majorization_book,
  title = {Inequalities: Theory of Majorization and Its Applications},
  ISBN = {9780387682761},
  ISSN = {2197-568X},
  url = {http://dx.doi.org/10.1007/978-0-387-68276-1},
  DOI = {10.1007/978-0-387-68276-1},
  journal = {Springer Series in Statistics},
  publisher = {Springer New York},
  author = {Marshall,  Albert W. and Olkin,  Ingram and Arnold,  Barry C.},
  year = {2011}
}

@misc{multimodal_fusion,
  doi = {10.48550/ARXIV.2411.17040},
  url = {https://arxiv.org/abs/2411.17040},
  author = {Li,  Songtao and Tang,  Hao},
  title = {Multimodal Alignment and Fusion: A Survey},
  publisher = {arXiv},
  year = {2024},
  copyright = {arXiv.org perpetual,  non-exclusive license}
}

@INPROCEEDINGS{reliability_fusion,
  author={Sidheekh, Sahil and Tenali, Pranuthi and Mathur, Saurabh and Blasch, Erik and Natarajan, Sriraam},
  booktitle={2024 27th International Conference on Information Fusion (FUSION)}, 
  title={On the Robustness and Reliability of Late Multi-Modal Fusion using Probabilistic Circuits}, 
  year={2024},
  volume={},
  number={},
  pages={1-8},
  doi={10.23919/FUSION59988.2024.10706372}}

@article{uncertainty_multimodal,
title = {A novel multi-modal fusion method based on uncertainty-guided meta-learning},
journal = {Pattern Recognition},
volume = {158},
pages = {110993},
year = {2025},
issn = {0031-3203},
doi = {https://doi.org/10.1016/j.patcog.2024.110993},
url = {https://www.sciencedirect.com/science/article/pii/S0031320324007441},
author = {Duoyi Zhang and Md Abul Bashar and Richi Nayak}
}

@misc{vte,
  doi = {10.48550/ARXIV.2405.14977},
  url = {https://arxiv.org/abs/2405.14977},
  author = {D\"{o}bler,  Mario and Marsden,  Robert A. and Raichle,  Tobias and Yang,  Bin},
  title = {A Lost Opportunity for Vision-Language Models: A Comparative Study of Online Test-Time Adaptation for Vision-Language Models},
  publisher = {arXiv},
  year = {2024},
  copyright = {arXiv.org perpetual,  non-exclusive license}
}

@misc{towards_robust_prompts,
  doi = {10.48550/ARXIV.2304.08479},
  url = {https://arxiv.org/abs/2304.08479},
  author = {Gu,  Jindong and Beirami,  Ahmad and Wang,  Xuezhi and Beutel,  Alex and Torr,  Philip and Qin,  Yao},
  title = {Towards Robust Prompts on Vision-Language Models},
  publisher = {arXiv},
  year = {2023},
  copyright = {arXiv.org perpetual,  non-exclusive license}
}

@misc{carot_vlm,
  doi = {10.48550/ARXIV.2311.01723},
  url = {https://arxiv.org/abs/2311.01723},
  author = {Oh,  Changdae and Lim,  Hyesu and Kim,  Mijoo and Han,  Dongyoon and Yun,  Sangdoo and Choo,  Jaegul and Hauptmann,  Alexander and Cheng,  Zhi-Qi and Song,  Kyungwoo},
  title = {Towards Calibrated Robust Fine-Tuning of Vision-Language Models},
  publisher = {arXiv},
  year = {2023},
  copyright = {Creative Commons Attribution 4.0 International}
}

@article{cpl_vlm,
  title = {Consistent prompt learning for vision-language models},
  volume = {310},
  ISSN = {0950-7051},
  url = {http://dx.doi.org/10.1016/j.knosys.2025.112974},
  DOI = {10.1016/j.knosys.2025.112974},
  journal = {Knowledge-Based Systems},
  publisher = {Elsevier BV},
  author = {Zhang,  Yonggang and Tian,  Xinmei},
  year = {2025},
  month = feb,
  pages = {112974}
}

@INPROCEEDINGS{craft_prompt_robust,
  author={Sun, Jingchen and Sharma, Rohan and Lokhande, Vishnu Suresh and Chen, Changyou},
  booktitle={2025 IEEE/CVF Winter Conference on Applications of Computer Vision (WACV)}, 
  title={Cross-Modal Feature Alignment and MMD Improve Robustness of Prompt Tuning}, 
  year={2025},
  volume={},
  number={},
  pages={4714-4724},
  doi={10.1109/WACV61041.2025.00462}}

@misc{provable_dynamic_fusion,
  doi = {10.48550/ARXIV.2306.02050},
  url = {https://arxiv.org/abs/2306.02050},
  author = {Zhang,  Qingyang and Wu,  Haitao and Zhang,  Changqing and Hu,  Qinghua and Fu,  Huazhu and Zhou,  Joey Tianyi and Peng,  Xi},
  title = {Provable Dynamic Fusion for Low-Quality Multimodal Data},
  publisher = {arXiv},
  year = {2023},
  copyright = {Creative Commons Attribution 4.0 International}
}

@ARTICLE{uncertainty_compensation_fusion,
  author={Papandreou, George and Katsamanis, Athanassios and Pitsikalis, Vassilis and Maragos, Petros},
  journal={IEEE Transactions on Audio, Speech, and Language Processing}, 
  title={Adaptive Multimodal Fusion by Uncertainty Compensation With Application to Audiovisual Speech Recognition}, 
  year={2009},
  volume={17},
  number={3},
  pages={423-435},
  doi={10.1109/TASL.2008.2011515}}

@book{entropy_majorization_book,
  title = {Entropy,  Divergence,  and Majorization in Classical and Quantum Thermodynamics},
  ISBN = {9789811666445},
  ISSN = {2197-1765},
  url = {http://dx.doi.org/10.1007/978-981-16-6644-5},
  DOI = {10.1007/978-981-16-6644-5},
  journal = {SpringerBriefs in Mathematical Physics},
  publisher = {Springer Singapore},
  author = {Sagawa,  Takahiro},
  year = {2022}
}

@article{zhou2022learning,
  title={Learning to prompt for vision-language models},
  author={Zhou, Kaiyang and Yang, Jingkang and Loy, Chen Change and Liu, Ziwei},
  journal={International journal of computer vision},
  volume={130},
  number={9},
  pages={2337--2348},
  year={2022},
  publisher={Springer}
}

@ARTICLE{lin1991js,
  author={Lin, J.},
  journal={IEEE Transactions on Information Theory}, 
  title={Divergence measures based on the Shannon entropy}, 
  year={1991},
  volume={37},
  number={1},
  pages={145-151},
  doi={10.1109/18.61115}}

@article{wang2026orca,
  title={Orca: The World is in Your Mind},
  author={Wang, Yihao and Ji, Yuheng and Cao, Mingyu and Shen, Yanqing and Xiao, Runze and Lyu, Huaihai and Xie, Senwei and Liu, Euan and Tian, Klara and Long, Tianfeng and others},
  journal={arXiv preprint arXiv:2606.30534},
  year={2026}
}

@article{COMBINER,
  title={COMBINER: Composed Image Retrieval Guided by Attribute-based Neighbor Relations},
  author={Li, Zixu and Hu, Yupeng and Chen, Zhiwei and Wen, Haokun and Song, Xuemeng and Nie, Liqiang},
  journal={IEEE TIP},
  year={2026},
  publisher={IEEE}
}

@inproceedings{ConeSep,
  title={Conesep: Cone-based robust noise-unlearning compositional network for composed image retrieval},
  author={Li, Zixu and Hu, Yupeng and Chen, Zhiwei and Zhang, Mingyu and Fu, Zhiheng and Nie, Liqiang},
  booktitle={Proceedings of the IEEE/CVF Conference on Computer Vision and Pattern Recognition},
  pages={16897--16909},
  year={2026}
}

@inproceedings{ReTrack,
  title={ReTrack: Evidence-Driven Dual-Stream Directional Anchor Calibration Network for Composed Video Retrieval},
  author={Li, Zixu and Hu, Yupeng and Chen, Zhiwei and Huang, Qinlei and Qiu, Guozhi and Fu, Zhiheng and Liu, Meng},
  booktitle={AAAI},
  volume={40},
  number={28},
  pages={23373--23381},
  year={2026}
}

\end{document}